\documentclass[conference]{IEEEtran}
\usepackage[left=0.68in, right=0.68in, bottom=1.03in, top=0.7in]{geometry}
\usepackage[ruled,vlined]{algorithm2e}

\usepackage{amsmath}
\usepackage{amsthm}
\usepackage{amsfonts}
\usepackage{amssymb}
\usepackage{bm}
\usepackage{boxedminipage}
\usepackage{cases}
\usepackage{array,color}
\usepackage{colortbl}
\usepackage{enumerate}
\usepackage{epsfig}
\usepackage{extarrows}
\usepackage{xcolor}
\usepackage{pifont}
\usepackage{bbm} 
\usepackage{cite}

\usepackage{booktabs}
\usepackage{multirow}
\usepackage{makecell}

\usepackage{graphicx}
\usepackage{subfigure}
\usepackage{dblfloatfix}
\usepackage{cuted} 
\graphicspath{{figures/}}

\newtheorem{theorem}{\textbf{Theorem}}

\newtheorem{lemma}{\textbf{Lemma}}

\newtheorem{remark}{\textbf{Remark}}

\newtheorem{assumption}{\textbf{Assumption}}

\newcommand{\etal}{\textit{et al. }}

\newcommand{\E}[1]{\mathbb{E}\left[#1\right]}
\newcommand{\norm}[1]{\left\|#1\right\|^2}
\newcommand{\inner}[2]{\left\langle #1, #2 \right\rangle}


\begin{document}
\IEEEoverridecommandlockouts

\title{EdgeFLow: Serverless Federated Learning\\via Sequential Model Migration in Edge Networks}

\author{\IEEEauthorblockN{
            Yuchen Shi\IEEEauthorrefmark{2}, Qijun Hou\IEEEauthorrefmark{2}, Pingyi Fan\IEEEauthorrefmark{1}\IEEEauthorrefmark{2},~\IEEEmembership{Senior Member,~IEEE,} and Khaled B. Letaief\IEEEauthorrefmark{4},~\IEEEmembership{Fellow,~IEEE} }
        
	\IEEEauthorblockA{
            \IEEEauthorrefmark{2}Dept. of Electronic Engineering, BNRist, Tsinghua University.\\
            \IEEEauthorrefmark{4}Dept. of Electronic and Computer Engineering, HKUST.}
    
        \thanks{\IEEEauthorrefmark{1}Pingyi Fan is the corresponding author, e-mail: fpy@tsinghua.edu.cn. This work was supported by the National Key Research and Development Program of China (Grant NO.2021YFA1000500(4)).}
        \thanks{\IEEEauthorrefmark{4}Khaled B. Letaief was supported in part by Hong Kong Research Grants Council under the Areas of Excellence scheme grant AoE/E-601/22-R.}
        }

\maketitle
\IEEEpeerreviewmaketitle

\begin{abstract}
    Federated Learning (FL) has emerged as a transformative distributed learning paradigm in the era of Internet of Things (IoT), reconceptualizing data processing methodologies. However, FL systems face significant communication bottlenecks due to inevitable client-server data exchanges and long-distance transmissions. This work presents EdgeFLow, an innovative FL framework that redesigns the system topology by replacing traditional cloud servers with sequential model migration between edge base stations. By conducting model aggregation and propagation exclusively at edge clusters, EdgeFLow eliminates cloud-based transmissions and substantially reduces global communication overhead. We provide rigorous convergence analysis for EdgeFLow under non-convex objectives and non-IID data distributions, extending classical FL convergence theory. Experimental results across various configurations validate the theoretical analysis, demonstrating that EdgeFLow achieves comparable accuracy improvements while significantly reducing communication costs. As a systemic architectural innovation for communication-efficient FL, EdgeFLow establishes a foundational framework for future developments in IoT and edge-network learning systems.
\end{abstract}


\section{Introduction}
\IEEEPARstart{T}{he} rapid advancement of Internet of Things (IoT) has led to exponential growth in edge-generated data and sensor network deployments, presenting fundamental challenges to traditional cloud-centric architectures. Therefore, developing efficient privacy-preserving mechanisms becomes imperative to enable secure and scalable distributed Artificial Intelligence (AI) across IoT ecosystems. Federated Learning (FL) has emerged as a transformative solution to this challenge, representing a paradigm-shifting approach that addresses these requirements by revolutionizing how distributed IoT data is processed and analyzed. First proposed by McMahan \etal in 2016 \cite{mcmahan2017communication}, FL is defined as a machine learning technique that enables multiple distributed clients to collaboratively train a model under the coordination of a central server, while keeping the data localized. With AI-enabled communication networks becoming a cornerstone of 6G systems \cite{letaief2019roadmap}, FL has demonstrated its significant potential in wireless communication, mobile edge computing and IoT ecosystem.

Unlike centralized learning that requires data pooling for model training, FL operates through the exchange of model parameters while maintaining data locality, consequently introducing non-negligible communication overhead in the learning system. In fact, communication efficiency has been widely identified as a critical bottleneck in FL systems, primarily due to the constrained and unreliable wireless links of local devices and sensors. Moreover, the intrinsic heterogeneity of both local data (non-IID distributions) and device capabilities in FL environments can significantly prolong training convergence and increase transmission latency. Furthermore, as the central server is typically not co-located with local devices, training results must traverse multiple edge nodes (base stations) via multi-hop routing before reaching the server. Although edge base stations and cloud servers are normally connected through high-speed backbone networks, the recurrent uploading and downloading of massive model parameters inevitably creates substantial packet queue load pressure.

To tackle this challenge, researchers have identified effective methods to reduce transmission volume per communication round, primarily through two strategies: 1) information compression via model selection \cite{nishio2019client, shi2024sam}, pruning \cite{zhu2023fedlp,jiang2022model} and quantization \cite{zhu2023towards}, and 2) efficient data utilization that trades computation resources for communication savings through techniques like local sampling \cite{karimireddy2020scaffold, zhu2024isfl}, coded computing \cite{shi2023fednc} and knowledge distillation \cite{li2019fedmd}. Alternatively, another line of research focuses on reducing communication overhead through network topology-aware aggregation mechanisms. A prominent example is Hierarchical FL \cite{liu2022hierarchical}, where geographically dispersed clients communicate with the cloud server through proximate edge nodes. Hierarchical FL organizes clients into multiple clusters, each managed by a dedicated edge node. During each communication round, the edge node first performs intra-cluster model aggregation, then transmits only the cluster-level model to the central server for global aggregation. This hierarchical approach eliminates the need for direct client-to-cloud transmissions, thereby reducing communication hops in the network.

However, even with Hierarchical FL which reduces communication frequency, the communication load between local devices and the cloud server persists. This challenge has driven the emergence of Sequential FL \cite{li2023convergence}, which eliminates the cloud server from the FL system, transforming it into a fully decentralized peer-to-peer (P2P) framework \cite{lalitha2018fully}. While it significantly reduces communication costs, the absence of parallel training introduces notable challenges including distribution shift and catastrophic forgetting \cite{kirkpatrick2017overcoming}. Some studies have attempted to integrate sequential training paradigms into hierarchical architectures by removing the cloud server and maintaining edge-based cluster training (rather than individual clients processing) \cite{yan2024sequential}, though these explorations remain limited in scope. Inspired by these developments, we propose the EdgeFLow framework, which replaces traditional cloud servers with sequential model migration flows across edge networks, effectively addressing the communication challenges.

The main contributions are summarized as follows:
\begin{itemize}
    \item We introduce the sequential model migration flows across edge networks into FL to replace traditional cloud servers and put forward EdgeFLow framework to mitigate the communication bottleneck challenge.
    \item We derive the convergence theorem for EdgeFLow by accounting for non-convex objective functions and non-IID local data distributions. The theoretical result extends the applicability of classical FL convergence analysis.
    \item We develop the algorithm and conduct various experiments to evaluate the performance of EdgeFLow. The outcomes demonstrate reasonable accuracy improvement alongside significant communication overhead reduction.
\end{itemize}

The rest of the paper is organized as follows. In Section \ref{sec:system}, we formalize the underlying problem in FL, present the core concept of EdgeFLow, and develop its algorithmic framework. In Section \ref{sec:theo}, we derive the convergence theorem with discussions. Section \ref{sec:exp} concentrates on the performance of the algorithm via numerical experiments, and finally, Section \ref{sec:con} concludes this work.

\section{System Model} \label{sec:system}
In this section, we first introduce the problem setup of FL system, then sketch the basic idea and the framework of EdgeFLow.

\subsection{Problem Setup}
Consider an FL system with $N$ clients, each with a fixed local data set $\mathcal{D}_n$, the basic FL problem is to minimize the global objective function, i.e.,
\begin{equation}
    \min_{\theta \in \mathbb{R}^d} \left\{ F(\theta) \triangleq \frac{1}{N} \sum_{n=1}^{N} f_n(\theta) \right\}
\end{equation}
where $F(\cdot)$, $f_n(\cdot)$, and $\theta$ denote the global objective function, the local objective function and the global model, respectively. In particular, when employing \textit{Stochastic Gradient Descent} that utilizes a randomly sampled mini-batch $\xi\subset\mathcal{D}_n$ per iteration, we typically consider $f_n(\theta)=\mathbb{E}_{\xi} \left[ \ell (\theta; \xi) \right]$, where $\ell (\cdot)$ is the loss function. 

In the traditional FL framework such as \textit{FedAvg} \cite{mcmahan2017communication}, during each communication round, a subset of clients is randomly selected from the $N$ total clients. These selected clients first replace their local models with the global model, then perform $K$ steps of local updates, and finally upload their respective local models to the central server for aggregation, generating a new global model. With SGD as the local optimizer, the local update rule follows:
\begin{equation} \label{eq:localupdate}
    \theta_{n,k+1}^{t} \gets \theta_{n,k}^{t} - \eta \tilde{g}_{n,k}^t
\end{equation}
where $t$, $k$, $n$ denote the number of communication round, local training step and the index of client, $\eta$ denotes the learning rate, and $\tilde{g}_{n,k}^t \triangleq \nabla \ell (\theta_{n,k}^{t}; \xi)$ denotes the stochastic gradient, respectively.

In conventional FL, frequent parameter exchanges between central servers and local clients are required. Given that modern model parameters are generally enormous in size as they scale proportionally with performance, this paradigm incurs substantial communication overhead. Consequently, reducing the massive data interactions between local devices and cloud server constitutes a critical solution to FL's communication bottleneck.

\begin{figure}[tbp]
    \centering
    \includegraphics[width=0.8\columnwidth]{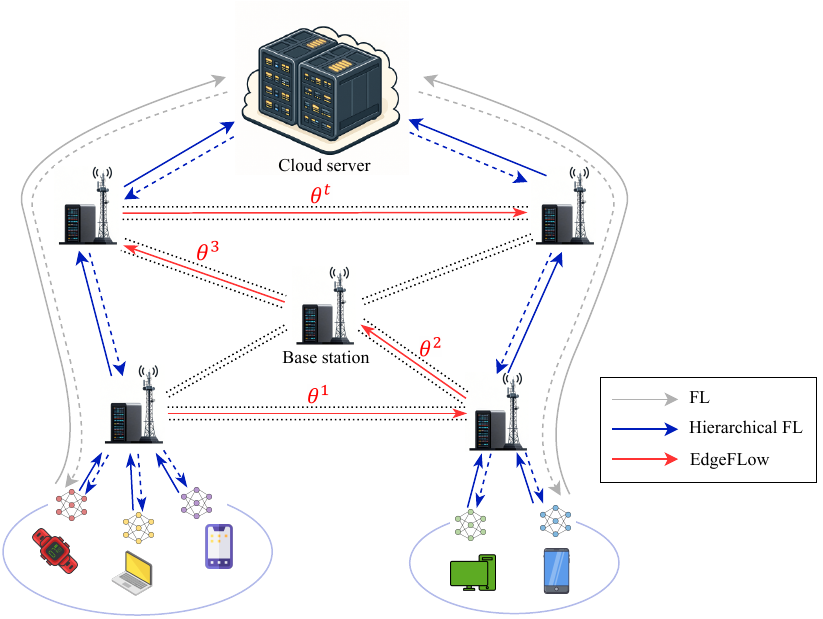}
    \caption{The comparison between FL, Hierarchical FL and EdgeFLow.} 
    \vspace{-0.5cm}
    \label{fig:edgeflow}
\end{figure}

\subsection{EdgeFLow Framework}
In traditional FL, locally trained models are uploaded to a cloud server for aggregation. However, there is no direct physical channel between local clients and the cloud server. Parameters are first transmitted to the nearest edge base station and then relayed through multiple hops before reaching the cloud. In practice, while local clients are geographically close to their edge base stations, the distance between base stations and the central server can be substantial, often spanning different cities or even countries.

In this work, we propose EdgeFLow, a novel approach that replaces the central cloud server with sequential model migration across edge networks. Technically, the framework operates in three phases:
\begin{enumerate}
    \item \textbf{Cluster Initialization}: Clients are dynamically grouped into fixed localized clusters, each cluster is anchored to an edge base station;
    \item \textbf{Intra-Cluster Training}: In each round, there is one cluster being active, within which all clients collaboratively train a model using their local data sets. Upon completing local training, the active base station performs model aggregation to update the global model;
    \item \textbf{Inter-Cluster Model Migration}: After completing the current round, the updated model directly migrates to the next scheduled cluster, where it is distributed to local clients for subsequent training iterations.
\end{enumerate}

Specifically, the $N$ clients in the FL system are partitioned into $M$ clusters. For the $t$-th round, let $m(t)$ denote the participating cluster, with its participating clients forming the set $\mathcal{N}_{m(t)}$, where the cardinality $|\mathcal{N}_{m(t)}|=N_{m(t)}$. Within each round, every client $n \in \mathcal{N}_{m(t)}$ conducts $K$ steps of local training to minimize its objective function $f_n(\theta_{n,k}^t)$, then uploads its parameter $\theta_{n,K}^t$ to the associated base station for aggregation. The base station updates the global model by
\begin{equation} \label{eq:globalupdate}
    \theta^{t+1}-\theta^t = -\frac{\eta}{N_{m(t)}} \sum_{n \in \mathcal{N}_{m(t)}} \sum_{k=0}^{K-1} \tilde{g}_{n,k}^t
\end{equation}
and then propagates it to the next scheduled edge node. The complete procedure is detailed in Algorithm 1.
\begin{algorithm}[htbp]
    \SetAlgoLined
    \KwIn{Local epoch $K$, global round $T$, learning rate $\eta$}
    \textbf{Initialization}\;
    \For{round $t=0, 1, \cdots, T-1$}{
        \For{each client $n\in \mathcal{N}_{m(t)}$ \textbf{in parallel}}{
            Download the global model $\theta^t$\;
            \For{step $k=0, 1, \cdots, K-1$}{
                Randomly sample a mini-batch $\xi\subset\mathcal{D}_n$\;
                Local training by Eq. \eqref{eq:localupdate}\;
                }
            Upload the model $\theta_{n, K}^t$ to the base station\;
        }
        the \textit{base station} \textbf{do}:\\
        Aggregate and update the global model by Eq. \eqref{eq:globalupdate}\;
        \uIf{$t<T-1$}{
            Select the next cluster $m(t+1)$ and transfer the model $\theta^{t+1}$ to the next base station\;
            }
        \lElse{\textbf{return} global model $\theta^{T}$}      
    }
    \caption{$\mathtt{EdgeFLow}$} \label{algo} 
\end{algorithm}

This edge-to-edge transfer forms a serverless model migration, bypassing the cloud entirely. The approach eliminates the costly transmission between edge nodes and a distant central server, as models now flow directly between edge base stations. 
Therefore, EdgeFLow mechanism effectively replaces the traditional cloud server's role while significantly reducing communication overhead.

\section{Theoretical Analysis} \label{sec:theo}
In this section, we establish the convergence guarantees for EdgeFLow under non-convex local objectives with non-IID data distributions, followed by a discussion of its physical interpretations and theoretical connections to classical convergence analysis.

First, we state the following assumptions that are commonly employed in convergence analyses within the existing literature. In particular, we relax the strongly or general convexity assumption on the local objective function, only assuming that each objective to be L-smooth in the non-convex case (Assumption \ref{ass:lsmooth}). Furthermore, we need to bound the norm and variance of the stochastic gradients to constrain the diversity during gradient descent (Assumption \ref{ass:gradient}). Finally, unlike existing works that often regulate the degree of non-IIDness in local data \cite{khaled2019first,li2019convergence}, we impose constraints on the heterogeneity within each cluster (Assumption \ref{ass:niid}). For analytical purpose, we define the cluster-level objective function $F_{m(t)}(\cdot)\triangleq \frac{1}{N_{m(t)}}\sum_{n\in\mathcal{N}_{m(t)}} f_n(\cdot)$.

\begin{assumption}[L-smoothness] \label{ass:lsmooth}
    The objective function $F(\cdot)$ is L-smooth, then there exists a Lipschitz constant $L$ such that
    \begin{equation}
        F(\theta) - F(\theta') \leq \inner{\nabla F(\theta') }{\theta-\theta'} + \frac{L}{2} \norm{\theta-\theta'}
    \end{equation} 
\end{assumption}

\begin{assumption}[Bounded Stochastic Gradient] \label{ass:gradient}
    The norm and variance of the stochastic gradients are bounded by $G^2$ and $\sigma^2$, respectively, to constrain their magnitude and variability.
    \begin{align}
        &\E{\norm{\tilde{g}_{n,k}^t}} \leq G^2 \\
        &\E{\norm{\tilde{g}_{n,k}^t - \E{\tilde{g}_{n,k}^t}}} \leq \sigma^2 
    \end{align}
\end{assumption}

\begin{assumption}[Heterogeneity] \label{ass:niid}
    The heterogeneity between the global data set and each cluster's data set is bounded by $\lambda_{m(t)}^2$, i.e.,
    \begin{equation}
        \E{\norm{\nabla F(\theta^t) - \nabla F_{m(t)}(\theta^t)}} \leq \lambda_{m(t)}^2
    \end{equation}
\end{assumption}

With the above assumptions, we are now ready to investigate the main theorem of the convergence analysis for EdgeFLow.
\begin{theorem} \label{th:conv}
    For EdgeFLow in Algorithm \ref{algo} with local epoch $K$, cluster size $N_{m(t)}$, and a given time range $T$, if the learning rate $\eta$ satisfies $LK\eta <1$, and the optimum of the expected objective function on the average model is denoted as $F^*$, then under the Assumption \ref{ass:lsmooth}-\ref{ass:niid}, the expectation of the squared gradient norm can be bounded by:
    \begin{equation}
        \begin{split}
            \frac{1}{T} \sum_{t=0}^{T-1} &\E{\norm{\nabla F(\theta^t)}} \leq \frac{4}{K\eta T}\left(F(\theta^0) - F^*\right) \\
            &+ \frac{2}{T} \sum_{t=0}^{T-1} \lambda_{m(t)}^2 + \frac{2}{T} \sum_{t=0}^{T-1} \frac{L\eta\sigma^2}{N_{m(t)}} + \frac{4L^2K^2\eta^2G^2}{3}
            \label{eq:th1}
        \end{split}
    \end{equation}
\end{theorem}

\begin{proof}
    For clarity in the proof, we first present the key lemmas to be employed, followed by the formal derivation.
    \begin{lemma} \label{lemma:1}
        Under Assumption \ref{ass:lsmooth}, it follows that
        \begin{equation} \label{eq:lemma:1}
            \begin{split}
                \E{\inner{\nabla F(\theta^t)}{\nabla F_{m(t)}(\theta^t)-\tilde{g}_{n,k}^t}} \qquad \qquad \qquad \qquad  \\
                \leq \frac{1}{4} \E{\norm{\nabla F(\theta^t)}} + \frac{L^2}{N_{m(t)}}\sum_{n \in \mathcal{N}_{m(t)}} \norm{\theta^t - \theta_{n,k}^t}
            \end{split}
        \end{equation}
    \end{lemma}

    \begin{lemma} \label{lemma:2}
        Under Assumption \ref{ass:niid}, it follows that
        \begin{equation} \label{eq:lemma:2}
            \begin{split}
                \mathbb{E}&\Big[\inner{\nabla F(\theta^t)}{\nabla F(\theta^t) - \nabla F_{m(t)}(\theta^t)} \Big] \\
                \leq& \frac{1}{2} \E{\norm{\nabla F(\theta^t)}} - \frac{1}{2} \E{\norm{\nabla F_{m(t)}(\theta^t)}} + \frac{1}{2}\lambda_{m(t)}^2 
            \end{split}
        \end{equation}
    \end{lemma}

    \begin{lemma} \label{lemma:3}
        Under Assumption \ref{ass:gradient}, it follows that
        \begin{equation} \label{eq:lemma:3}
            \frac{1}{N_{m(t)}}\sum_{n \in \mathcal{N}_{m(t)}} \norm{\theta^t - \theta_{n,k}^t} \leq k^2 \eta^2 G^2
        \end{equation}
    \end{lemma}

    According to Assumption \ref{ass:lsmooth}, if $F(\cdot)$ is L-smooth, then we have the following inequality:
    \begin{align}
        \mathbb{E} \big[ & F(\theta^{t+1})-F(\theta^{t}) \big] \nonumber \\
        & \leq \E{\inner{\nabla F(\theta^t)}{\theta^{t+1} - \theta^{t}}} + \frac{L}{2} \E{\norm{\theta^{t+1} - \theta^{t}}} \\
        \begin{split}
            & = -\frac{\eta}{N_{m(t)}} \sum_{n \in \mathcal{N}_{m(t)}} \sum_{k=0}^{K-1} \E{\inner{\nabla F(\theta^t)}{\tilde{g}_{n,k}^t}} \\
            & \qquad \qquad \quad + \frac{L\eta^2}{2{N_{m(t)}}^2} \E{\norm{\sum_{n \in \mathcal{N}_{m(t)}} \sum_{k=0}^{K-1} \tilde{g}_{n,k}^t}} 
        \end{split}
        \label{eq:lsmooth}
    \end{align}
    
    First, we bound the first term in Eq. \eqref{eq:lsmooth}. Through the decomposition of inner products, we have
    \begin{align} \label{eq:th1pf1}
        \begin{split}
            &-\frac{\eta}{N_{m(t)}} \sum_{n \in \mathcal{N}_{m(t)}} \sum_{k=0}^{K-1} \E{\inner{\nabla F(\theta^t)}{\tilde{g}_{n,k}^t}} \\
            =&\frac{\eta}{N_{m(t)}} \sum_{n} \sum_{k} \mathbb{E} \Big[ \inner{\nabla F(\theta^t)}{\nabla F_{m(t)}(\theta^t) - \tilde{g}_{n,k}^t} \\
            + & \left\langle \nabla F(\theta^t), \nabla F(\theta^t) - \nabla F_{m(t)}(\theta^t) \right\rangle - \norm{\nabla F(\theta^t)} \Big]
        \end{split}
    \end{align}
    Based on the conclusions of Lemma \ref{lemma:1} and Lemma \ref{lemma:2}, we substitute Eq. \eqref{eq:lemma:1} and \eqref{eq:lemma:2} into Eq. \eqref{eq:th1pf1}, yielding:
    \begin{equation}
        \begin{split}
            &- \frac{\eta}{N_{m(t)}}  \sum_{n \in \mathcal{N}_{m(t)}} \sum_{k=0}^{K-1} \E{\inner{\nabla F(\theta^t)}{\tilde{g}_{n,k}^t}} \\
            \leq & -\frac{K\eta}{4} \E{\norm{\nabla F(\theta^t)}} - \frac{\eta}{2} \sum_{k} \E{\norm{\nabla F_{m(t)}(\theta^t)}} \\
            & + \frac{K\eta}{2} \lambda_{m(t)}^2 + \frac{L^2 \eta}{{N_{m(t)}}^2} \sum_{n} \sum_{k} \sum_{n} \norm{\theta^t - \theta_{n,k}^t} 
        \end{split} \label{eq:1term}
    \end{equation}
    
    Next, we bound the second term in Eq. \eqref{eq:lsmooth}. We have
    \begin{align}
        & \frac{L\eta^2}{2{N_{m(t)}}^2} \E{\norm{\sum_{n \in \mathcal{N}_{m(t)}} \sum_{k=0}^{K-1} \tilde{g}_{n,k}^t}} \nonumber \\
        \begin{split}
            &\quad \overset{\text{\ding{192}}}{=}  \frac{L\eta^2}{2{N_{m(t)}}^2} \sum_{n} \sum_{k} \E{\norm{  \tilde{g}_{n,k}^t - \E{\tilde{g}_{n,k}^t} }} \\
            &\quad \qquad + \frac{LK^2\eta^2}{2} \E{\norm{ \frac{1}{N_{m(t)}K} \sum_{n} \sum_{k} \E{\tilde{g}_{n,k}^t} }}
        \end{split}
        \\
        &\quad \overset{\text{\ding{193}}}{\leq} \frac{LK\eta^2 \sigma^2}{2N_{m(t)}} + \frac{LK\eta^2}{2} \sum_{k} \E{\norm{ \nabla F_{m(t)}(\theta^t)}} \label{eq:2term}
    \end{align}
    where \ding{192} holds due to the fact that $\mathbb{E} \norm{X} = \mathbb{E}\norm{X - \mathbb{E} X} + \norm{ \mathbb{E}X} $ and the zero-mean property of $\tilde{g}_{n,k}^t - \mathbb{E}\tilde{g}_{n,k}^t$, \ding{193} follows by applying Assumption \ref{ass:gradient} and Jensen's Inequality.
    
    With upper bounds established for both the first and second terms in Eq. \eqref{eq:lsmooth}, we can substitute Eq. \eqref{eq:1term} and \eqref{eq:2term} back into Eq. \eqref{eq:lsmooth}, resulting in
    \begin{align} \label{eq:th1pf2}
        \begin{split}
            & \mathbb{E} \big[ F (\theta^{t+1})-F(\theta^{t}) \big] \leq -\frac{K\eta}{4} \E{\norm{\nabla F(\theta^t)}} \\
            & + \frac{K\eta}{2} \lambda_{m(t)}^2  + \frac{L^2 \eta}{{N_{m(t)}}^2} \sum_{n} \sum_{k} \sum_{n} \norm{\theta^t - \theta_{n,k}^t} \\
            & + \frac{LK\eta^2 \sigma^2}{2N_{m(t)}} + \frac{\eta}{2}\left( LK\eta - 1 \right) \sum_{k} \E{\norm{ \nabla F_{m(t)}(\theta^t)}}
        \end{split}
    \end{align}
    Since $LK\eta < 1$, the term of $\nabla F_{m(t)}(\theta^t)$ in Eq. \eqref{eq:th1pf2} can be omitted. According to Lemma \ref{lemma:3}, we can substitute Eq. \eqref{eq:lemma:3} into Eq. \eqref{eq:th1pf2} and obtain
    \begin{equation}
        \begin{split} 
        \mathbb{E} \big[ F (\theta^{t+1}) & - F(\theta^{t}) \big] \leq -\frac{K\eta}{4} \E{\norm{\nabla F(\theta^t)}} \\
        &+ \frac{K\eta}{2} \lambda_{m(t)}^2 + \frac{LK\eta^2 \sigma^2}{2N_{m(t)}} + \frac{L^2 K^3 \eta^3 G^2}{3}
        \end{split}
        \label{eq:th1pf3}
    \end{equation}
    which follows from the fact that $\sum_{k=0}^{K-1}k^2 \leq \frac{K^3}{3}$.
    
    Consider the time average of $\frac{1}{T}\sum_{t=0}^{T-1}[\cdot]$ on both sides of Eq. \eqref{eq:th1pf3} and denote the optimum of $F(\theta^T)$ as $F^*$, we have
    \begin{align}
        \begin{split}
            \frac{1}{T} & \mathbb{E}[F^* - F(\theta^0)] \leq -\frac{K\eta}{4T} \sum_{t=0}^{T-1} \E{\norm{\nabla F(\theta^t)}} \\
            &+ \frac{K\eta}{2T} \sum_{t=0}^{T-1} \lambda_{m(t)}^2 + \frac{1}{T} \sum_{t=0}^{T-1} \frac{LK\eta^2 \sigma^2}{2N_{m(t)}} + \frac{L^2 K^3 \eta^3 G^2}{3}
        \end{split}
    \end{align}
    which leads to the final conclusion.
\end{proof}

\begin{remark}
    Theorem \ref{th:conv} clearly demonstrates how EdgeFLow affects convergence performance by replacing the central server with sequential model migration at the edge level. The upper bound in Eq. \eqref{eq:th1} consists of four key terms: 1) the initialization gap between the starting and optimal objectives, 2) the data heterogeneity bias induced by non-IID client distributions, 3) the gradient variance resulting from model aggregations, and 4) the training error from stochastic mini-batch sampling. Additionally, other hyperparameters may also impact convergence behavior, such as the local epoch $K$ and the cluster size $N_{m(t)}$.
\end{remark}

The theoretical result extends the applicability of classical FL convergence analysis, particularly regarding our proposed cluster-based data heterogeneity bound $\lambda_{m(t)}^2$. Two special cases require attention:
\begin{itemize}
    \item \textbf{IID Scenario}: For any cluster $m(t)$, we may approximate $\nabla F(\theta^t) = \nabla F_{m(t)}(\theta^t)$, i.e., $\lambda_{m(t)}^2 = 0$. The convergence bound thus simplifies to
    \begin{equation}
        \mathcal{O} \left( \frac{F(\theta^0) - F^*}{K\eta T}+  \frac{L\eta\sigma^2}{N_{m(t)}} + L^2K^2\eta^2G^2 \right)
        \label{eq:th1iid}
    \end{equation} 
    which is consistent with established literature \cite{yu2019parallel}.
    \item \textbf{Traditional FL Scenario}: In the traditional setting, all participating clients can be viewed as forming a temporary cluster during each training round. Unlike EdgeFLow where the number of clusters $m$ remains fixed, and thereby enables better control over the range of $\lambda_{m(t)}^2$, the traditional FL approach suffers from indeterminate cluster counts and greater randomness, consequently making control over the bound of $\lambda_{m(t)}^2$ more difficult.
\end{itemize}

\section{Experiment Results} \label{sec:exp}
In this section, we present experimental evaluations, assessing the accuracy improvements and communication efficiency of EdgeFLow across multiple task configurations. The evaluation provides empirical validation of our theoretical analysis.

\subsection{Basic Settings}
Our experiments utilize the FashionMNIST as well as CIFAR-10 benchmark datasets for 10-class image classification. The learning model implements a six-layer convolutional neural network (CNN) with $3 \times 3$ kernels. Each convolutional layer incorporates batch normalization, while $2 \times 2$ max-pooling operations follow every second convolutional layer to reduce spatial dimensions. The architecture concludes with two fully connected layers with $(128, 10)$ output features for classification, trained end-to-end using cross-entropy loss and the Adam optimization algorithm.

The FL system organizes $N=100$ local clients into $M$ fixed clusters, i.e., $N_{m(t)}=\frac{100}{M}$ clients per cluster. In contrast, FedAvg, serving as the conventional FL baseline, randomly samples $N_{m(t)}$ clients every training round without maintaining persistent cluster memberships. Selected clients perform $K$ local epochs per communication round using a mini-batch size of $64$. 

The client data follows two distinct distribution settings, \textit{IID} and \textit{$x\%$-non-IID}. For the \textit{IID} setting, each client is randomly and uniformly assigned image samples from all $10$ categories. For the \textit{$x\%$-non-IID} setting, each client's dataset contains one or two major categories (accounting for $x\%$ of samples). Specifically, our experiments employ: 
\begin{enumerate}
    \item IID configuration consisting of:
    \begin{itemize}
        \item $100$ clients with \textit{IID} setting.
    \end{itemize} 
    \item NIID A configuration consisting of:
    \begin{itemize}
        \item $10$ clients with \textit{IID} setting;
        \item $20$ clients with \textit{$95\%$-non-IID} setting;
        \item $70$ clients with \textit{$98\%$-non-IID} setting.
    \end{itemize} 
    \item NIID B configuration consisting of:
    \begin{itemize}
        \item $10$ clients with \textit{IID} setting;
        \item $90$ clients with \textit{$100\%$-non-IID} setting.
    \end{itemize} 
\end{enumerate} 
These three distribution configurations are visually depicted in Fig. \ref{fig:config}.
\begin{figure}[htbp]
    \centering
    \subfigure[IID]{
    \label{fig:iid}
    \includegraphics[width=0.31\columnwidth]{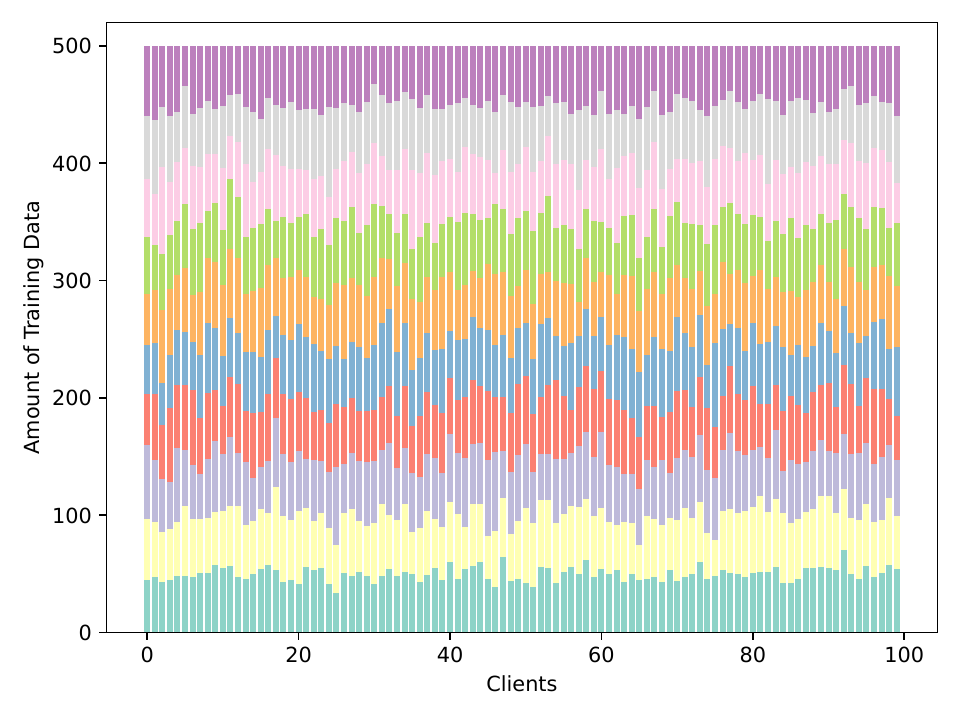}}
    \subfigure[NIID A]{
    \label{fig:niida}
    \includegraphics[width=0.31\columnwidth]{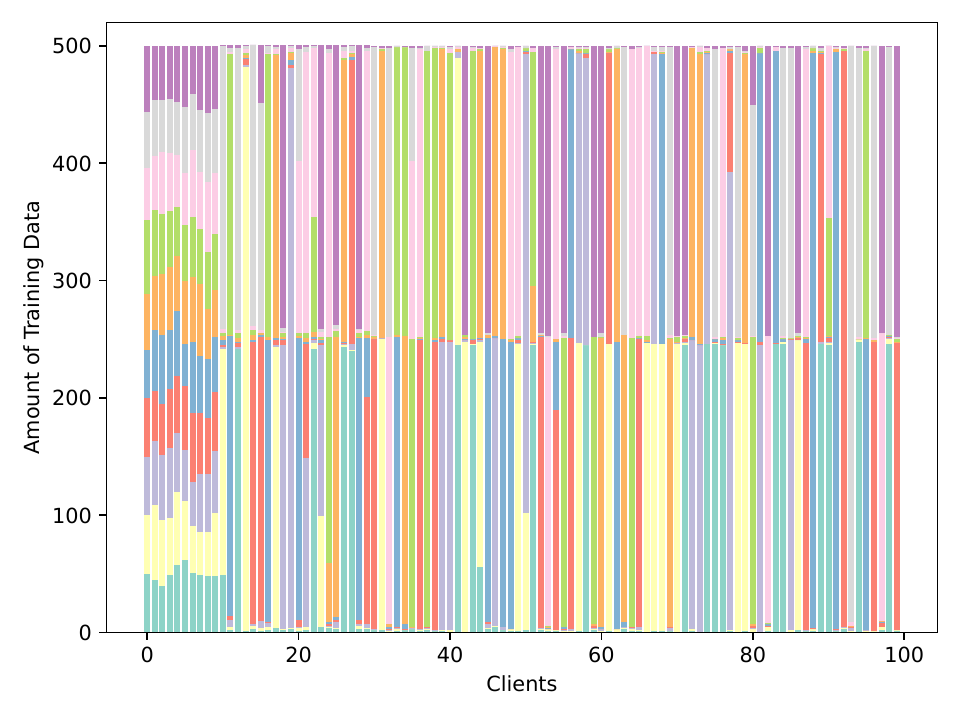}}
    \subfigure[NIID B]{
    \label{fig:niidb}
    \includegraphics[width=0.31\columnwidth]{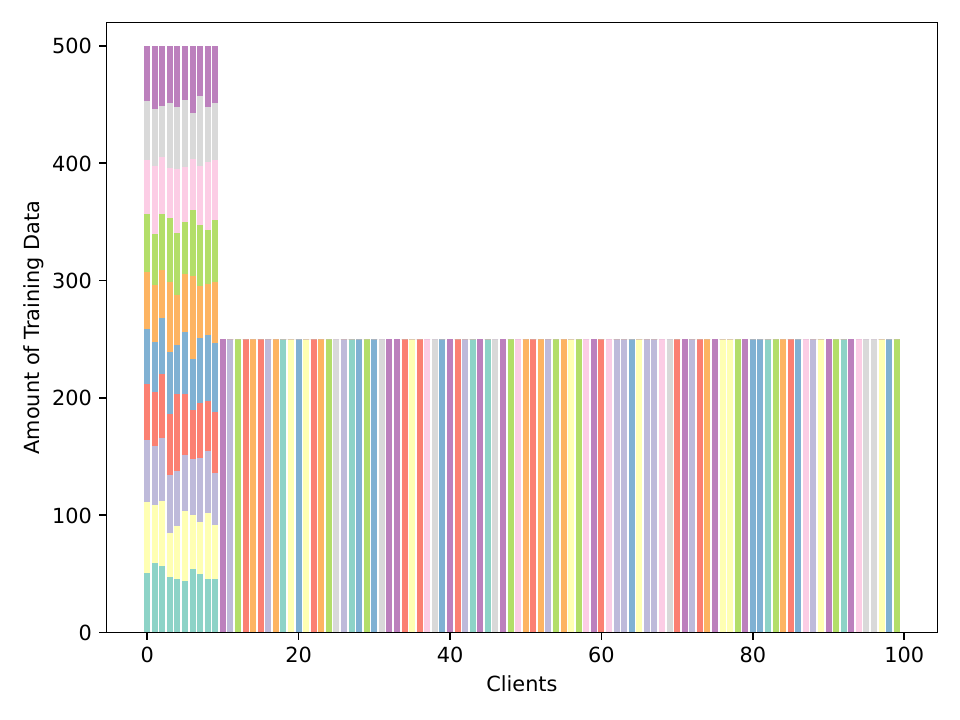}}
    \caption{Comparative demonstration of three data configuration paradigms. NIID A features distribution skewness, whereas NIID B features both distribution and quantity skewness.}
    \label{fig:config}
\end{figure}

\subsection{Accuracy Performance}
We evaluate the accuracy performance under two datasets with three configurations. In addition to the baseline experiments \textit{FedAvg}, we designed two variants of EdgeFLow: \textit{EdgeFLowRand} and \textit{EdgeFLowSeq}. The key distinction lies in the cluster selection strategy after each training round, EdgeFLowRand randomly selects the next cluster, whereas EdgeFLowSeq follows a predetermined fixed sequence.

In our experiments, we set $N_{m(t)}=10$ and $K=5$. The detailed results are presented in Table \ref{tab}. It can be observed that under the IID configuration, as the convergence bound degenerates to Eq. \eqref{eq:th1iid}, the performance difference between EdgeFLow and FedAvg is marginal. However, under the NIID configurations, EdgeFLow achieves consistent improvements across both datasets, with particularly notable enhancements observed on the more complex dataset (CIFAR-10).

\begin{table}[hb]
    \caption{Results of Accuracy Performance ($\%$)}
    \label{tab}
    \centering
    \begin{tabular}{cccccc}
        \toprule
        {\multirow{3}{*}{Method}} & \multicolumn{2}{c}{FashionMNIST} & \multicolumn{3}{c}{CIFAR-10} \\
        \cmidrule(l){2-3} \cmidrule(l){4-6}
         & IID & NIID A & IID & NIID A & NIID B\\
        \midrule
        FedAvg & \textbf{90.60} & 86.89 & 88.66 & 77.04 & 71.04 \\
        EdgeFLowRand & 90.13 & \textbf{87.97} & \textbf{89.16} & 80.26 & 73.14 \\
        EdgeFLowSeq & 90.53 & 87.50 & 88.99 & \textbf{81.58} & \textbf{73.36} \\
        \bottomrule
    \end{tabular}
\end{table}

In addition, Fig. \ref{fig:plot} illustrates EdgeFLow's performance evolution under NIID B configuration across varying cluster sizes $N_m$ and local epochs $K$. Fig. \ref{fig:n} demonstrates that larger $N_m$ values lead to simultaneously faster convergence and higher accuracy, consistent with Theorem \ref{th:conv}'s theoretical predictions. Meanwhile, Fig. \ref{fig:k} reveals that increasing $K$ does not yield proportional improvements in model performance, since $K$ appears in both the numerator and denominator terms in Eq. \eqref{eq:th1} of Theorem \ref{th:conv}. This dual presence results in a non-monotonic convergence bound, necessitating empirical methods to determine the optimal local epoch setting.

\begin{figure}[thb]
    \centering
    \subfigure[Cluster size $N_m$]{
    \label{fig:n}
    \includegraphics[width=0.44\columnwidth]{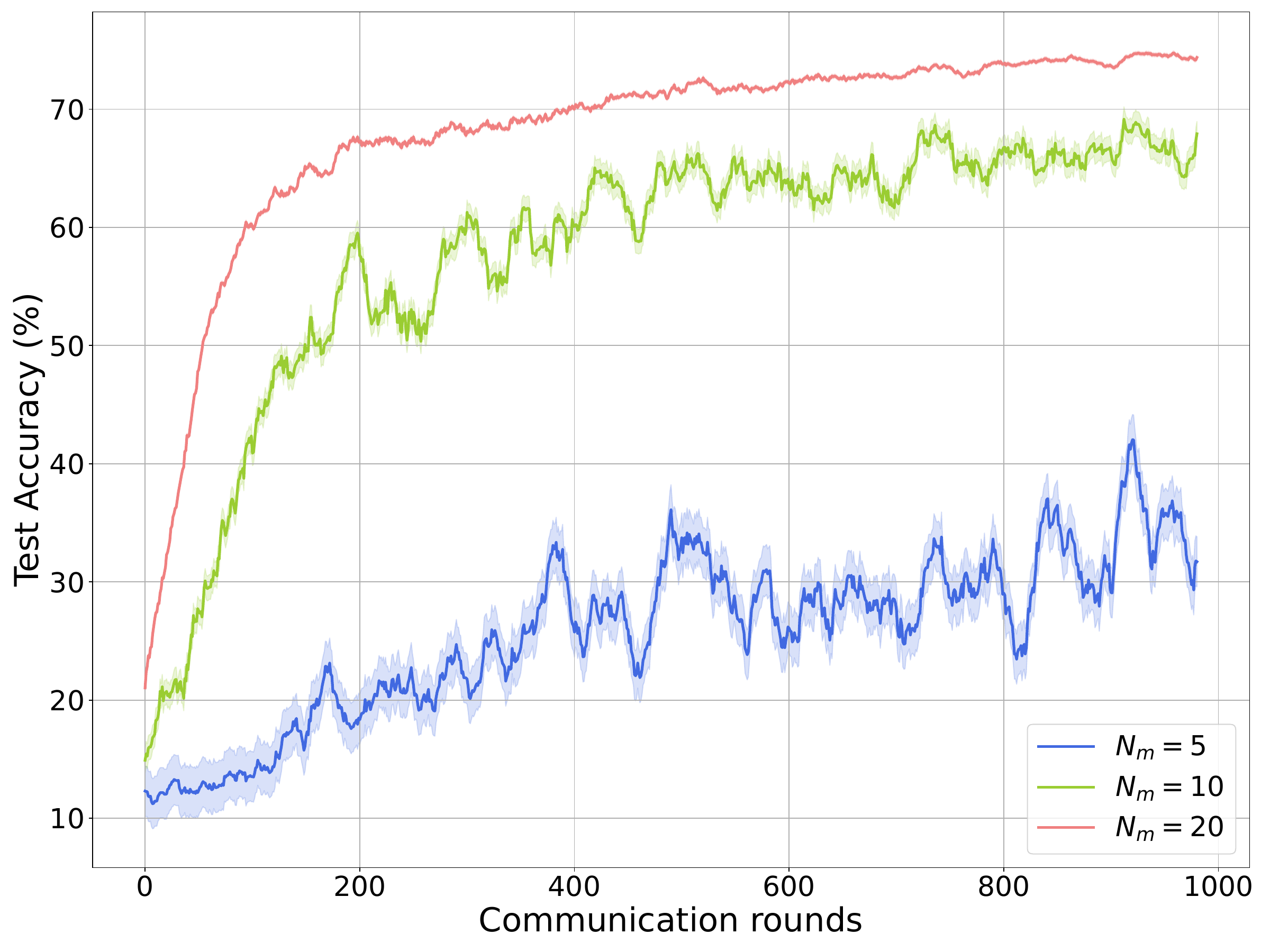}}
    \subfigure[Local epoch $K$]{
    \label{fig:k}
    \includegraphics[width=0.44\columnwidth]{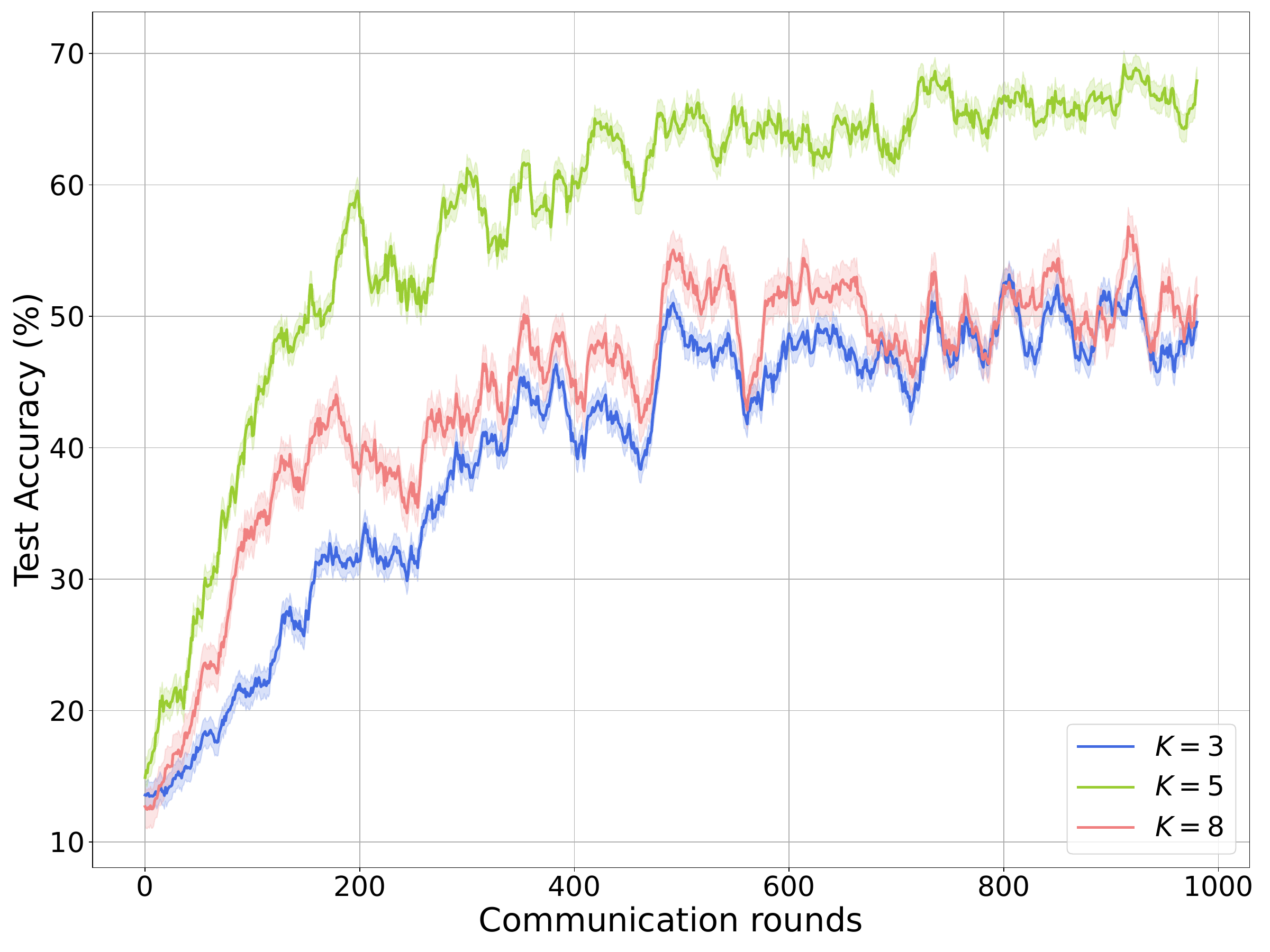}}
    \caption{Impact of hyperparameter variations on EdgeFLow. Accuracy curves are smoothed with a sliding window for visualization.}
    \label{fig:plot}
\end{figure}

\subsection{Communication Savings}
We further evaluated EdgeFLow's communication efficiency. The communication load is measured by the count of parameters uploaded per round, and the compression ratio is defined as the amount of data EdgeFLow needs to transmit divided by the original data transmission amount, where a lower value indicates better compression effectiveness. Specifically, FedAvg and Hierarchical FL account for parameter transmission from local devices to the cloud server, and EdgeFLow considers transmission from local devices to the next base station. We analyzed four common edge network topologies: 1) simple (local-edge-cloud), 2) breadth-parallel, 3) depth-linear, and 4) hybrid breadth-depth complex structures. As shown in Fig. \ref{fig:load}, EdgeFLow demonstrates increasing communication efficiency gains as the topology grows more complex, especially when the distance between local devices and cloud servers increases, which leads to more transmission hops. Compared to "breadth-oriented" cases, EdgeFLow shows notably better optimization effects for such "depth-oriented" scenarios.
\begin{figure}[htbp]
    \centering
    \includegraphics[width=0.85\columnwidth]{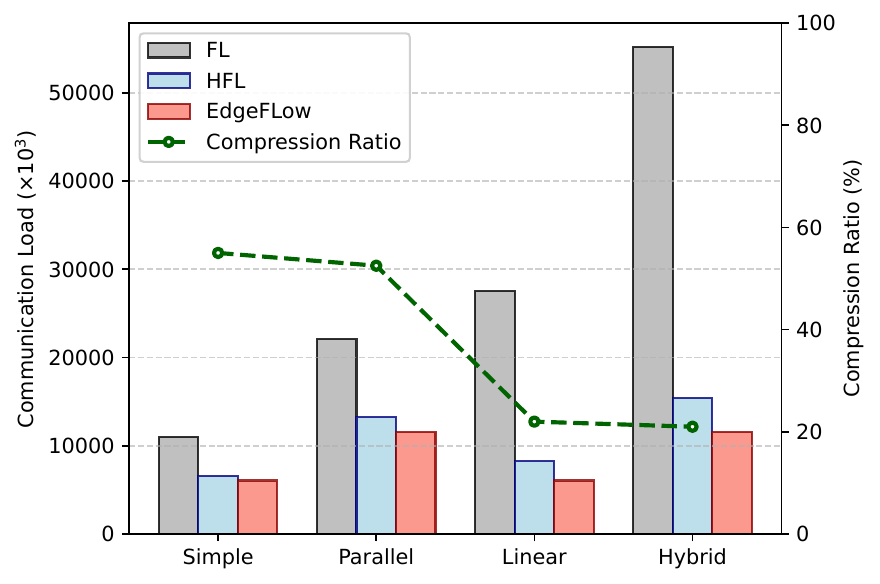}
    \caption{Communication load comparison across network structures.} 
    \label{fig:load}
\end{figure}

\section{Conclusion} \label{sec:con} 
In this work, we proposed EdgeFLow, a novel FL framework that replaces traditional cloud-centric aggregation with sequential model migration across edge networks. By eliminating the need for centralized cloud servers, EdgeFLow significantly reduces communication overhead while maintaining competitive model performance. Our convergence analysis demonstrated that EdgeFLow achieves provable guarantees under non-convex objectives and non-IID data distributions, extending classical FL theory to edge-centric topologies. Experimental results on benchmark datasets confirmed that EdgeFLow significantly reduces communication costs by 50-80\% compared to conventional FL approaches, while achieving comparable accuracy.

This work establishes a new direction for communication-efficient FL, proving that edge-based training can overcome the fundamental bottlenecks of cloud-dependent systems. Future research may explore dynamic cluster formation for adaptive edge learning and the integration with wireless-aware scheduling. EdgeFLow provides a foundational architecture for scalable FL in next-generation IoT and mobile computing systems, and it may also possess greater potential capabilities in terms of safety and personalization.

\bibliographystyle{IEEEtran}
\bibliography{ref}

\end{document}